%% file: main.tex
\documentclass[12pt]{article} % Include author names

\usepackage[margin=1in]{geometry}

\usepackage{natbib}
\usepackage{url}

\usepackage{times}

\usepackage[utf8]{inputenc} % allow utf-8 input
\usepackage[T1]{fontenc}    % use 8-bit T1 fonts
\usepackage{hyperref}       % hyperlinks
\usepackage{url}            % simple URL typesetting
\usepackage{booktabs}       % professional-quality tables
\usepackage{amsfonts}       % blackboard math symbols
\usepackage{nicefrac}       % compact symbols for 1/2, etc.
\usepackage{microtype}      % microtypography
\usepackage{diagbox}

\usepackage{dsfont}
\usepackage{amsmath}
\usepackage{amsthm}
\usepackage{hyperref}
\usepackage[pdftex]{graphicx}
\hypersetup{breaklinks=true}

\usepackage{float}
\usepackage{algorithm}
\usepackage{algorithmic}
\usepackage{amssymb}

\usepackage{enumerate}

\newtheorem{theorem}{Theorem}
\newtheorem{prop}{Proposition}
\newtheorem{dfn}{Definition}

\newtheorem{lem}{Lemma}
\newtheorem{bib_rem}{Bibliographic remark}

 \newcommand{\norm}[1] {\left| \left| #1  \right| \right|}
 \newcommand{\RR} {\mathbb{R}}

\title{Generalized mean shift \\ with triangular kernel profile}
\author{S. Razakarivony, A. Barrau 
\thanks{S. Razakarivony and A. Barrau are with  SAFRAN TECH, Groupe Safran, Rue des Jeunes
Bois - Ch\^ateaufort, 78772 Magny Les Hameaux CEDEX, France {\tt\small sebastien.razakarivony@safrangroup.com axel.barrau@safrangroup.com}. A. Barrau is associate researcher at MINES
ParisTech, PSL Research University, Centre for
robotics, 60 Bd St Michel 75006 Paris, France {\tt\small axel.barrau@mines-paristech.fr}}. }

\begin{document}

\maketitle

\begin{abstract}

The mean shift algorithm is a popular way to find modes of some probability density functions taking a specific kernel-based shape, used for clustering or visual tracking. Since its introduction, it underwent several practical improvements and generalizations, as well as deep theoretical analysis mainly focused on its convergence properties. In spite of encouraging results, this question has not received a clear general answer yet. In this paper we focus on a specific class of kernels, adapted in particular to the distributions clustering applications which motivated this work. We show that a novel Mean Shift variant adapted to them can be derived, and proved to converge after a finite number of iterations. In order to situate this new class of methods in the general picture of the Mean Shift theory, we alo give a synthetic exposure of existing results of this field.

%In spite of encouraging results, this question is not solved for all kernels yet. Moreover, classical mean shift relies on the Euclidean norm and little work has been carried out on extending the algorithm to other metrics better suited to specific applications. This kind of generalizations is made possible by recent research casting the mean shift algorithm into the framework of bound optimization. Our contributions is threefold: first, we make a small review of mean shift theory. Second, we extend the convergence results to allow new variants of mean shift. Third, we introduce two applications of our convergence result : median shift and wasserstein median shift, two new variants of mean shift using respectively L1 and Wasserstein distance.

\end{abstract}

\input{1-Intro.tex}

\input{2-Overview_Mean_Shift.tex}

\input{3-MeanShiftTriangular.tex}

\input{4-Results.tex}

\section{Conclusion and perspectives}

In this work, after a quick review of the results on mean shift theory, we introduced a new class of algorithm with convergence properties. We proposed two applications : the median shift algorithm introduced in \cite{shapira2009mode}  to which we give a convergence proof, and a novel algorithm : the Wasserstein median shift. We also provide some experiments to show the usefullness of the latter.
In further works, we intend to look at extensions of the algorithms, to obtain clustering with different bandwith regarding the samples, and the possibility to extend to other profile than the triangular profile.

\bibliographystyle{plain}

\end{document}

\title{Median shift: a L1-wasserstein clustering algorithm, supplementary materials}

\maketitle

We illustrate here by more results the median shift algorithm. We performed two more experiments: one on synthetic data, with a pure illustrative purpose, and one on real aeronautical data. 

\begin{figure}
\center
\includegraphics[width=12cm]{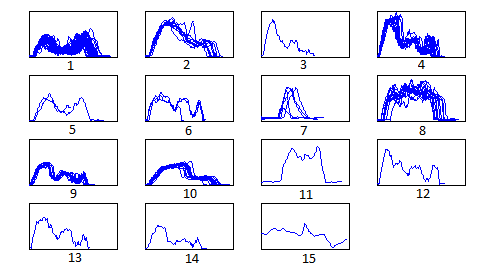}
\includegraphics[width=12cm]{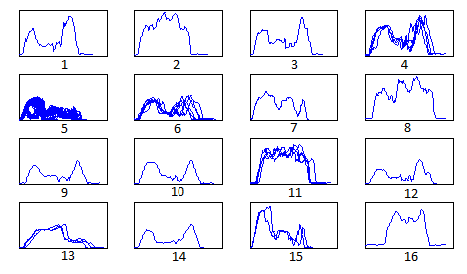}
\caption{Clusters returned by two algorithms. Top: Wasserstein median shift. We see that cluster 1,2,3,4,5,6,7 of Figure \ref{fig::trajdata} have been correctly recovered, respectively, as cluster 8, 9, 1, 10, 4, 2 and 7. The outliers have caused the apparition of eight additional clusters: 3, 5, 6, 11, 2, 3, 14, 15. Bottom: classical mean shift. Only clusters 1 and 4 of Figure \ref{fig::trajdata} seem partially recovered, and cluster 5, containing almost all the histograms,has been created.}
\label{fig::trajdata_res}
\end{figure}

\begin{table}
\center
\begin{tabular}{|c|c|c|}
\hline
Method 			& 	Adjusted Rand Index (Simulated)	    & Adjusted Rand Index (Real)\\
\hline
Affinity Propagation 	& 		0.24					& 0.49				  \\
\hline
Spectral Clustering 	& 		0.11					& 0.44				  \\
\hline
Ward 	& 							0.0					& 0.41				  \\
\hline
Agglomerative Clustering 	&		0.0					& 0.0				  \\
\hline
DBSCAN 			& 	0.0									& 0.0				  \\
\hline
Gaussian Mixture&	0.97								& 0.41				  \\
\hline
kmeans L2 		& 	0.0									& 0.41				  \\
\hline
Birch 		& 	0.0										& 0.41				  \\
\hline
kmeans WS 		&										& 0.74				  \\
\hline
mean shift L2 	& 	0.3315 								& 0.15				  \\
\hline
median shift 	&	1.0									& 0.78 				  \\

\hline
\end{tabular}
\caption{Results for the simulated dataset (left) and the trajectory dataset (right).}
\label{tab::tab1}
\end{table}
We worked on public climate data from the American National Oceanic and Atmospheric Administration \footnote{The data we used are available at \url{https://www.ncdc.noaa.gov/monitoring-references/maps/us-climate-divisions.php}}. For each state (except Alaska and Hawaii for which data are missing), we built the histogram of rainfall and of temperatures over the years as a feature. Then, we used these features to look for groups of states having similar behavior. We tested classical mean shift and our Wasserstein median median shift. We cannot compute the ARI indicator because there is no ground truth in this case,  but we can look at the obtained results and qualitatively assess their relevance with respect to the different climates present in the USA. The results of a clustering based on temperatures are given on Figure \ref{fig::rainfalls}. We see that Wasserstein median shift, as expected, isolates California and splits the north and south part of the east coast, which is not achieved by classical mean shift.

\begin{figure}
\center
\begin{tabular}{cc}
\includegraphics[width=5.0cm]{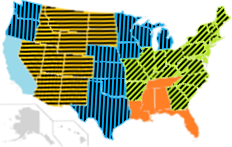} & 
\includegraphics[width=5.0cm]{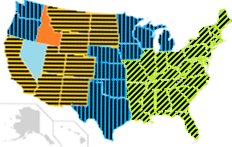} \\
Wasserstein median shift & Mean shift
\end{tabular}
\caption{Climate in the USA, clustered applying two algorithms to rainfall histograms of the different states. Left: Wasserstein median shift ; right: mean shift.}
\label{fig::rainfalls}
\end{figure}
On Figure \ref{fig::temperatures} the clustering is based on temperatures. The results are not as clear as with rainfalls, both methods returning horizontal slices more due to global temperature increase from north to south than to complicated histogram shape variations.

\begin{figure}
\center
\begin{tabular}{cc}
\includegraphics[width=5.0cm]{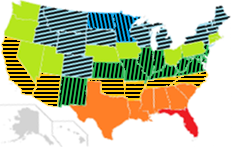} & 
\includegraphics[width=5.0cm]{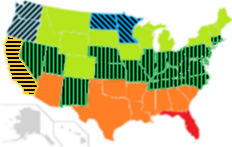} \\
Wasserstein median shift & Mean shift
\end{tabular}
\caption{Climate in the USA, clustered applying two algorithms to temperature histograms of the different states. Left: Wasserstein median shift ; right: mean shift.}
\label{fig::temperatures}
\end{figure}

 The histograms obtained for the seven categories of missions are displayed on Figure \ref{fig::trajdata}. We see that classical mean shift is not adapted to the problem: it gathers almost all flights in cluster 5. Conversely, all the true clusters are among those returned by the Wasserstein median shift. The ARI obtained for the mean shift is 0.15, as the ARI of our algorithm is 0.78. Even if the result is not perfect, our algorithm clearly outperform the classical mean shift in this practical case.

\bibliography{mainbib}
\bibliographystyle{plain}

%% file: 1-Intro.tex
\section{Introduction}

The mean shift algorithm is a simple iterative algorithm introduced in \citep{fukunaga1975estimation}, which became over years a cornerstone of data clustering.
Technically, its purpose is to find local maxima of a function having the following shape:
\begin{equation}
\label{eq::ker_density}
f(x)=\frac{1}{Nh^q} \sum_{i}k \left(\frac{\norm{x-x_{i}}^2}{h} \right),
\end{equation}
where $(x_i)_{1 \leq i \leq N}$ is a set of data points belonging to a vector space $\RR^q$, $h$ is a scale factor, $k(.)$  is a convex and decreasing function from $\RR_{\geq 0}$ to $\RR_{\geq 0}$ and $\norm{\cdot}$ denotes the Euclidean norm in $\RR^q$.
This optimization is done by means of a sequence $n \rightarrow \hat{x}_n$, which we will explicite in Sect. \ref{sect::Mean_Shift_theory}.
(To alleviate notations, the factor $ \frac{1}{Nh^q} $ in Eq. \eqref{eq::ker_density} will be always omitted in this paper). Usually, functions such as $f(.)$ of Eq. \eqref{eq::ker_density} are built to approximate an underlying probability density function (p.d.f.) from which the vectors $(x_i)_{1 \leq i \leq N}$ are assumed to be sampled. Following the terminology used in the kernel density estimation theory, $K(x,y) = k\left(\frac{\norm{x-y}^2}{h} \right)$ is the kernel while $k(.)$ is refered to as the \emph{profile} of this kernel.

As a mode seeking procedure, Mean Shift can be used for clustering: \emph{high density} areas surrounding the modes of $f(\cdot)$ can be regarded as clusters of the point cloud $(x_i)_{1 \leq i \leq N}.$ The advantage of this approach, and more generally of density-based methods lies in the reduced tuning work they require compared to other classical clustering algorithms. In partitioning-based clustering for example, to which belong the very popular K-means algorithm and its variations \citep{KMEANS,KMEANS_V1,KMEANS_V2,KMEANS_V3}, the number of clusters has to be set by the user. Hierarchical methods (such as OPTICS \citep{OPTICS} or CURE \citep{CURE}) return a dendrogram conveying the structure of the data set, not directly turnkey clusters. Model-based methods need even more input from the user as they assume the data follow some specific model characterized by a set of parameters to be estimated. Some well known examples of such approaches are Expectation Maximization (EM) \citep{EM} or Self-Organizing Map \citep{SOM}. To come back to density-based methods, they do resort to prior knowledge on the structure of data through a choice of kernel metric. In particular, setting a proper scaling requires having an idea of the order of magnitude of the distance between points inside a cluster. The other way to introduce human expertise into Eq. \eqref{eq::ker_density} is using a meaningful distance with respect to the nature of the data.

If we can notice some works that extend Mean Shift to manifold space using kernels, as in \citep{vedaldi2008quick} or \citep{cetingul2009intrinsic}, this point is usually not addressed directly in the literature, as the distance used has to be Euclidean (possibly up to a variable change). However, the reformulation of Mean Shift as a bound optimization, proposed in \citep{fashing2005mean}, suggests an interesting generalization to any (possibly non-Euclidean) distance. 
Note that convergence Mean Shift for a general kernel, even based on the Euclidean distance, is still an open problem as the proof given in \citep{MEANSHIFT} has been contested in \citep{li2007note} and in \citep{ghassabeh2015sufficient}. The closest result to the one we will show is, \citep{huang2018convergence} which focuses on the more specific case of Epanechnikov mean shift. In the present paper, we clarify what is meant by ``convergence" claimed by the different articles and what results have been proved, and situate our contribution in this general picture. Our contribution is threefold:
\begin{enumerate}
\item We make a synthetic review of mean shift theory.
\item We prove a new convergence result for the case of triangular kernel profiles, an immediate application of which is showing convergence of the Median Shift algorithm proposed in \citep{shapira2009mode}.
\item We propose a novel convergent mode seeking algorithm we call \emph{Wasserstein Median Shift}, adapted to clustering of histogram representing distributions and which was the initial motivation for the present work.
\end{enumerate}

The remaining of the article is organized as follows. In Section \ref{sect::Mean_Shift_theory} we review the theory of mean shift, with a small focus on the generalization suggested in \citep{fashing2005mean}. In Section \ref{sect::flat}, the core of our contribution, we study the specific case of the triangular profile and prove convergence of the returned sequence $n \rightarrow \hat{x}_n$, apply it to Median Shift and introduce the Wasserstein Median Shift algorithm. Finally, in Section \ref{sect::Results} we illustrate the theoretical results of the article on a synthetic and a real dataset.

%% file: 2-Overview_Mean_Shift.tex
\section{An overview of the mean shift theory}
\label{sect::Mean_Shift_theory}

This section summarizes the main results of the theory of mean shift. In \ref{sect::MS_interpretations}, the equations of the algorithm are given along with their different interpretations. In \ref{sect::generalized_MS}, the equations of a generalized mean shift suggested in \citep{fashing2005mean}, applicable to non-euclidean distances, are derived. Finally, the main convergence results of the mean shift theory are listed in \ref{sect::CV_MS}.

\subsection{The mean shift algorithm}
\label{sect::MS_interpretations}

As explained in the introduction, the mean shift algorithm's purpose is to find local maxima of the function defined by Eq.\eqref{eq::ker_density}.
To do so,  given a point $x$ of the data space, the algorithm builds a sequence $\hat{x}_0, \hat{x}_1, \hat{x}_2, \dots$ starting at $\hat{x}_0=x$ that converges to a local maximum of the density function. This is achieved iterating the following step: 
\begin{equation}
\label{eq::Mean_Shift}
\hat{x}_{n+1}=\frac{\sum_{i}{\alpha_i x_{i}}}{{\sum_{i=1}^N{\alpha_i}}},
\end{equation}
with $\alpha_i = g(||\frac{\hat{x}_n-x_{i}}{h}||^2)$
where $g(\cdot)=-k'(\cdot)$ is the opposite of the derivative of $k(\cdot)$. This process, the convergence of which to a mode of $f(.)$ will be discussed in Section \ref{sect::CV_MS}, can be given three interpretations: as a weight average, as a gradient ascend and as a bound optimization.

\paragraph{Mean shift as a weighted average}

Each mode estimate $\hat{x}_n$ is a weighted average of the elements of the data set, with $\alpha_i$ the weight assigned to the element $x_i$. The profile $k(.)$ being decreasing and convex, the function $g(.)=-k'(.)$ appearing in the definition of the weights $\alpha_i$ is positive and decreasing: the further a point is from the current average, the lower its weight will be in the computation of the next average. Although insightful, this interpretation does not really help understanding why this process of local averaging would converge to a mode of the function defined by Eq. \eqref{eq::ker_density}.

\paragraph{Mean shift as a gradient ascend}

Another interpretation is based on the remark that Eq. \eqref{eq::Mean_Shift} can be re-written as:
$$
\hat{x}_{n+1} = \hat{x}_n + \lambda(\hat{x}_n) \nabla f|_{\hat{x}_n},
$$
with $\lambda(\hat{x}_n)= 1/\sum_{i=1}^N \alpha_i$ and $\nabla f|_{x} = -\sum_i \alpha_i (x-x_i)$ is the gradient of $f(.)$ (checking this result is a direct computation).
From this point of view mean shift is a gradient ascend with a step size $\lambda(\hat{x}_n)$ updated at each iteration. 

\paragraph{Mean shift as a bound optimization}
The third interpretation, proposed in \citep{fashing2005mean}, assimilates mean shift to a bound optimization algorithm.
The term \emph{bound optimization} denotes a wide range of algorithms consisting in replacing a function $f(\cdot)$ to be optimized with an approximation $\hat{f}_n(\cdot)$, having properties ensuring that optimizing $\hat{f}_n(\cdot)$ will increase the value of $f(\cdot)$ \citep{hunter2004tutorial}. Once the optimum of $\hat{f}_n(\cdot)$ has been found, it is used to build a new approximation $\hat{f}_{n+1}(\cdot)$ and the operation is iterated until convergence of the sequence. 
Seeing mean shift from this viewpoint, we can write Eq. \eqref{eq::Mean_Shift} as:
\begin{equation}
\label{eq::bound}
\hat{x}_{n+1} = \text{argmax} \hat{f}_n,
\end{equation}
where the function $\hat{f}_n$ is defined as
\begin{equation}
\label{eq::def_fn}
\hat{f}_n(x) = f(\hat{x}_n) -  \sum_i g \left( \left|\left| \frac{\hat{x}_n-x_i}{h} \right|\right|^2 \right) \left( \left| \left| \frac{x-x_i}{h} \right | \right|^2 - \left| \left| \frac{\hat{x}_n-x_i}{h} \right | \right|^2 \right).
\end{equation}
%This reformulation is in the spirit of \citep{fashing2005mean}, but is not directly taken from it.
It can be easily checked that zeroing the gradient of $\hat{f}_n(\cdot)$ in Eq. \eqref{eq::def_fn} ends in Eq. \eqref{eq::Mean_Shift}, i.e., the classical formulation of mean shift. This point of view allowed the authors of \citep{fashing2005mean} to revisit the main result of the mean shift theory (growth of the sequence $n \rightarrow f(\hat{x}_n)$) in an especially insightful way. As noticed in this reference, the approach can be generalized to build a mean shift variant based on any (possibly non-Euclidean) distance. To our best knowledge the authors never wrote down nor used the algorithms obtained that way but their derivation, to which Sect. \ref{sect::generalized_MS} below is dedicated, is a necessary step to the new results we will prove in Sect. \ref{sect::flat}.

\subsection{Mean shift based on a general distance}
\label{sect::generalized_MS}

\begin{bib_rem}
The ideas of this section are mentioned in the Discussion and Conclusions section of \citep{fashing2005mean}. What follows was (to our best knowledge) never written in black and white, and their suggestion is not clear wether it is the Euclidean norm or the squared Euclidean norm that could easily be changed. However, what follows should be considered in all honesty a part of existing theory and a starting point for the results of Section \ref{sect::flat}, not a contribution of the present paper.
\end{bib_rem}

A benefit of the reformulation of mean shift as a bound optimization appears if we replace the squared Euclidean norm in Eq. \eqref{eq::ker_density} with a general function $d(\cdot,\cdot)$ on the state space $\mathbb{R}^q$ going to $+\infty$ when the norm of its first argument goes to infinity. The counterpart of Eq. \eqref{eq::ker_density} is then a p.d.f. of the form (omitting the factor $\frac{1}{Nh^q}$):
\begin{equation}
\label{eq::f_general}
f(x)= \sum_{i} k \left( d \left( x,x_i \right) \right).
\end{equation}
There is \emph{a priori} no obvious transposition of Eq. \eqref{eq::Mean_Shift} in this case. Conversely, adapting the bound optimization equations \eqref{eq::bound}, \eqref{eq::def_fn} is straightforward and leads to the following iterative process:
\begin{equation}
\label{eq::iterate2}
\hat{x}_{n+1} \in \text{argmax} \hat{f}_n,
\end{equation}
with $\hat{f}_n$ defined as:
\begin{equation}
\label{eq::def_fn2}
\hat{f}_n(x) = f(\hat{x}_n) - \sum_i g \left( d \left( \hat{x}_n-x_i \right) \right) \Big[ d \left( x,x_i \right) - d \left( \hat{x}_n,x_i \right) \Big].
\end{equation}
Note that taking the minimum of $-\hat{f}_n$ in \eqref{eq::iterate2} (instead of the maximum of $\hat{f}_n$) and removing the constant terms in \eqref{eq::def_fn2}, does not change the position of the optimum and leads to a simpler formulation of the  same algorithm:
\begin{equation}
\label{eq::iterate3}
\hat{x}_{n+1} \in \underset{x}{\text{argmin}} \sum_i \alpha_i d \left( x,x_i \right), \quad \text{with} \quad \alpha_i = g( d \left( \hat{x}_n-x_i \right)) .
\end{equation}

This generalization will be leveraged in Section \ref{sect::WMS} to build new versions of mean shift, better suited in particular to clustering histograms representing distributions and having convergence guarantees.
For now, we can notice that classical properties of bound optimization hold, in particular Proposition \ref{prop::growth_generalized} below. 

\begin{prop}
\label{prop::growth_generalized}
The sequence $n \rightarrow f(\hat{x}_n)$, with $\hat{x}_n$ verifying \eqref{eq::iterate3}, is growing. Moreover, as it is also upper bounded, it is convergent.
\end{prop}
\begin{proof}
Using Lemma \ref{prop::f_hat_n} (see below) we have: $f(\hat{x}_{n+1}) \geq \hat{f}_n(\hat{x}_{n+1}) \geq \hat{f}_n(\hat{x}_n) = f(\hat{x}_n)$.
\end{proof}
The proof of Proposition \ref{prop::growth_generalized} used the following lemma:
\begin{lem}
\label{prop::f_hat_n}
The functions $f(\cdot)$ of Eq. \eqref{eq::f_general} and  $\hat{f}_n(\cdot)$ of Eq. \eqref{eq::def_fn2} verify the following properties:
\begin{enumerate}[(i)]
\item $\hat{f}_n(\hat{x}_n) = f(\hat{x}_n)$ \label{prop::f_hat_n_1}
\item $\hat{f}_n (\hat{x}_n) \leq \hat{f}_n (\hat{x}_{n+1})  $ \label{prop::f_hat_n_2}
\item $\forall x, \hat{f}_n(x) \leq  f(x) $ \label{prop::f_hat_n_3}
\end{enumerate}
\end{lem}

\begin{proof}
\eqref{prop::f_hat_n_1} is immediate using definition \ref{eq::def_fn2}.
\eqref{prop::f_hat_n_2} is a straightforward consequence of \eqref{eq::iterate2}.
\eqref{prop::f_hat_n_3} can be obtained as follows. The function $ k(.) $ being convex we have $k(u) \geq k(u_0)+k'(u_0)\cdot (u-u_0)$ for any $u \geq 0$. Replacing $u_0$ with $ \frac{d(\hat{x}_n,x_i)}{h}$ and $u$ with $ \frac{d(x-x_i)}{h} $ in the latter expression then summing over $i$ we get for any $x$:
\begin{align*}
f(x) & \geq \sum_i k \left( \frac{d(\hat{x}_n,x_i)}{h} \right)\\
& +  k' \left( \frac{d(\hat{x}_n,x_i)}{h} \right) \left( \frac{d(x,x_i)}{h} - \frac{d(\hat{x}_n,x_i)}{h} \right) \\
 & \geq f(\hat{x}_n) -   \sum_i^N g \left( \frac{d(\hat{x}_n,x_i)}{h} \right) \left( \frac{d(x,x_i)}{h} - \frac{d(\hat{x}_n,x_i)}{h} \right).
\end{align*}
The last expression being exactly the function $\hat{f}_n(.)$ defined by Eq. \eqref{eq::def_fn} we obtain \eqref{prop::f_hat_n_3}.
%Regarding, \ref{prop::f_hat_n_4}, $\hat{f}_n(\cdot)$ feing defined in \eqref{eq::def_fn} as a linear combination of concavefunctions it is itself a concave function. Finally, \ref{prop::f_hat_n_5} is a straightforward computation.
\end{proof}

Note that Prop. \ref{prop::growth_generalized} does not imply convergence of the mode estimates $\hat{x}_n$, which has to be studied on a case-by-case basis. The existing convergence results for main variants of the Mean Shift algorithm will be summed up in Sect. \ref{sect::CV_MS}.

\subsection{Convergence of mean shift algorithms}
\label{sect::CV_MS}

In the previous subsection we generalized the main convergence result of the theory of mean shift (Prop. \ref{prop::growth_generalized}), verified by all variants of the method. It concerns $f(\hat{x}_n)$, not the mode estimate $\hat{x}_n$ itself. In this subsection, we give a quick overview of the results obtained in the litteratture on the convergence properties of mean shift algorithms. We propose to classify these methods depending on their convergence guarantees. Three guarantee levels can usually be shown, from the lower to the higher:
\begin{enumerate}
\item The sequence $n \rightarrow f(\hat{x}_n)$ is non-decreasing \label{prop::non-decreasing}
\item The sequence $n \rightarrow \hat{x}_n$ converges to a stationary point of $f$
\item The sequence $n \rightarrow \hat{x}_n$ converges to a local maximum of $f$
\end{enumerate}
We have of course $3 \Rightarrow 2 \Rightarrow 1$. $1$ is the minimal result of a mean-shift-based algorithm. $2$ ensures we can use the algorithm as a clustering method: once we know the process is convergent we can define a cluster as a set of initializations $x$ bringing the sequence $\hat{x}_n$ to the same limit value. $3$ has been shown for Euclidean norm and Gaussian kernel in \citep{ghassabeh2015sufficient} and for Euclidean norm and Epanchenikov kernel in \citep{huang2018convergence}. The current state of the theory of mean shift can be summed up by the following table, where the numbers 1, 2, 3 represent what result have been shown for each situation. Our contribution is to prove $2$ for flat $g(.)$ for any distance (circled 2 in the table).

\begin{center}
\begin{tabular}{|c|c|c|c|}
\hline
\backslashbox{$d(\cdot,\cdot)$}{$g(\cdot)$}     &   exponential  &         flat                & general smooth \\
\hline
Squared Euclidean &   1, 2, 3     &            1, 2, 3                  &     1, 2, 3      \\
\hline
General   &     1       &   1, \large{\textcircled{\small{2}}} &       1 \\
\hline
\end{tabular}

\end{center}

%% file: 3-MeanShiftTriangular.tex
\section{Mean shift with triangular profile}
\label{sect::flat}

This section contains the theoretical contributions of the present paper. The main result, established in Section \ref{sect::Main}, is the convergence of the generalized mean shift of Sect. \ref{sect::generalized_MS} when the weighting function $g(.)$ is flat (i.e. the profile $k(.)$ is triangular). Actually, the sequence $n \rightarrow \hat{x}_n$ will even be shown to be stationary. In Sect. \ref{sect::Mean_Shift_L1} and Sect. \ref{sect::WMS} two clustering algorithms are derived from this result: Median Shift (already proposed in \citep{shapira2009mode} and now given a convergence proof) and Wasserstein median shift. The latter, which is the initial motivation for the work presented in this paper, is particularly well suited to histogram clustering.

\input{3-0-MainContribution.tex}
\input{3-1-L1_Median_Shift.tex}

\input{3-2-Wasserstein_Median_Shift.tex}

%% file: 3-0-MainContribution.tex
\subsection{Main result}
\label{sect::Main}

In the present section, we derive the generalized mean shift equations in the case where the profile $k(.)$ is triangular, and prove they generate a convergent sequence. We have here:
%$$
%k(u) = \left \{ \begin{matrix} 1-u & \text{if} \quad u < 1 \\
% 0 & \text{if} \quad u \geqslant 1
%\end{matrix} \right.
%$$
%$$
%g(u) = \left \{ \begin{matrix} 1 & \text{if} \quad u < 1 \\
% 0 & \text{if} \quad u \geqslant 1
%\end{matrix} \right.
%$$

\begin{equation}
k(u) = \left \{ \begin{matrix} 1-u & \text{if} \quad u < 1 \\
0 & \text{if} \quad u \geqslant 1
\end{matrix} \right.
\quad \text{and} \quad
g(u) = \left \{ \begin{matrix} 1 & \text{if} \quad u < 1 \\
 0 & \text{if} \quad u \geqslant 1
\end{matrix} \right.
\end{equation}

Applying formula \eqref{eq::iterate3} we obtain the following iteration:
\begin{equation}
\label{eq::mu}
\hat{x}_{n+1} = \mu(C_n) \in \underset{x}{\text{argmin}} \sum_{i \in C_n} d(x,x_i)
\end{equation}
where $\mu(\cdot)$ is a choice function that returns always the same $x$ for a given $C_n$, and $C_n = \{ i, | d(x_n, x_i) < 1 |\}$. 
For a given $n$, we call $\{x_i, i \in C_n\}$ the active set.
The only benefit in introducing $\mu$ in \eqref{eq::mu} instead of writing $\hat{x}_{n+1} = \underset{x}{\text{argmin}} \sum_{C_n} d(x, x_i)$ is to ensure that taking twice the same set $C_n$ will give the same $\hat{x}_{n+1}$ in case several minimizers exist.
Now we can come to the main result of this section:

\begin{theorem}
\label{thm::main}
The sequence $\hat{x}_n$ defined by see Eq. \eqref{eq::mu} (i.e. generalized mean shift in the case of flat weights), is stationnary regardless of the distance used.
It converges in a finite number of step.
\end{theorem}

The rest of this section will be dedicated to proving this result. Let us start with ``candidates'':

\begin{dfn}[Candidates]
We define as ``candidates'' all points $x$ verifying $x = \mu(J)$ for $J$ a subset of indices. As there is a finite number of subsets $J$ in the dataset, there is a finite number of candidates.
\end{dfn}

Obviously, each $\hat{x}_n$ is a candidate, which means $\hat{x}_n$ evolves inside a finite set, as well as $f(\hat{x}_n)$, its image through $f$. Using Prop. \ref{prop::growth_generalized} we also know that $f(\hat{x}_n)$ is growing, which implies:
\begin{prop}
\label{prop::f_stat}
The sequence $n \rightarrow f(\hat{x}_n)$ is stationnary.
\end{prop}

Prop. \ref{prop::f_stat} does not imply that $n \rightarrow \hat{x}_n$ is also stationary, as it could endlessly shift betwen different values having the same image through $f(.)$. We need the following technical proposition:

\begin{prop}
\label{prop::strictly}
At a step $n \geqslant 1$, if a point is added to the active set, then $n \rightarrow f(\hat{x}_n)$ strictly grows even if some other points are removed. More formally, if (at least) one data point index $i \not \in C_{n}$ and $i \in C_{n+1}$ then we have $f(\hat{x}_n) < f(\hat{x}_{n+1})$.
\end{prop}

\begin{proof}
%At eah step $n$ we define $u_{n+1} = \sum_{C_{n+1}\C_n} [1-d(\hat{x}_{n+1},x_i)]$ and $d_{n+1} = -\sum_{C_n\C_{n+1}} [1-d(\hat{x}_{n+1},x_i)]$.
The set $C_{n+1} \cup C_n$ can be cut as $C_{n+1} \cup \left[ C_n \setminus C_{n+1} \right]$ or as $C_n \cup \left[C_{n+1} \setminus C_n \right]$. Splitting the sum $\sum_{C_{n+1}\cup C_n} [1-d(\hat{x}_{n+1},x_i)]$ over these two decompositions of the set $C_{n+1} \cup C_n$ we obtain the equality:
$$
\underset{\boxed{1}}{\sum_{C_{n+1}} [1-d(\hat{x}_{n+1},x_i)]} + 
\underset{\boxed{2}}{\sum_{C_n \setminus C_{n+1}} [1-d(\hat{x}_{n+1},x_i)]} =
\underset{\boxed{3}}{\sum_{C_n} [1-d(\hat{x}_{n+1},x_i)]} +
\underset{\boxed{4}}{ \sum_{C_{n+1} \setminus C_n} [1-d(\hat{x}_{n+1},x_i)]}
$$
Let us study each term. We have the following results:
\begin{itemize}
\item $\boxed{1}=\hat{f}_{n+1}(\hat{x}_{n+1})$. \emph{Proof: by definition of $\hat{f}_{n+1}$.}
\item $\boxed{2} \leqslant 0$. \emph{Proof: by definition of $C_{n+1}$ we have $[1-d(\hat{x}_{n+1},x_i)]\leqslant 0$ for any $i \not \in C_{n+1}$.}
\item $\boxed{3} \geqslant \hat{f}_{n}(\hat{x}_{n})$. \emph{Proof: By definition, $\hat{x}_{n+1}$ minimizes $\sum_{C_n}d(x,x_i)$ over all possible values of $x$, so that we have in particular $\sum_{C_n} d(\hat{x}_{n+1},x_i) \leqslant \sum_{C_n} d(\hat{x}_n,x_i)$, thus $\sum_{C_n} [1-d(\hat{x}_{n+1},x_i)] \geqslant \sum_{C_n} [1-d(\hat{x}_n,x_i)] = \hat{f}_n(\hat{x}_n)$}
\item $\boxed{4} \geqslant 0$, with zero obtaind only if $C_{n+1} \setminus C_n$ is empty. \emph{Proof: by definition of $C_n$ we have $[1-d(\hat{x}_n,x_i)] > 0$ for any $i \in C_n$.}
\end{itemize}

In the case where $C_{n+1} \setminus C_n$ is empty we have $\boxed{4} > 0$ and $\hat{f}_{n+1}(\hat{x}_{n+1}) = \boxed{1} > \boxed{3} \geqslant \hat{f}_n(\hat{x}_n)$

\end{proof}

An immediate consequence of Proposition \ref{prop::strictly} is that if $f(\hat{x}_n)$ is non-increasing after a given step $p_0$ then no data point can be added to $C_n$ for $n>p_0$, i.e. we have $C_{n+1} \subset C_n$ for $n>p_0$. In particular, the cardinal of $C_n$ is non-increasing after rank $p_0$. But this cardinal is a positive integer, so it is necessarily stationary after a rank $p_1>0$. For $n \geqslant p_1 \geqslant p_0$ no point can be added to nor removed from $C_n$ and this set becomes stationary, in a finite number of steps. Remembering Eq. \eqref{eq::mu} we conclude that $\hat{x}_n$ is stationary and prooves of Theorem \ref{thm::main}.

The application which motivated the present work is clustering of histograms, for which the Wasserstein metric defined in Section \ref{sect::WMS} by Equation \eqref{eq::wasserstein_metric} is much more meaningful than the Euclidean distance. We will see that working with the Wasserstein metric can be reduced to working with the $L_1$-norm, so we first analyze this case in Section \ref{sect::Mean_Shift_L1}.

%% file: 3-1-L1_Median_Shift.tex
\subsection{Application 1: Median shift}
\label{sect::Mean_Shift_L1}

In this subsection, we present median shift, a L1 variant of mean shift. 
The algorithm was already presented in \citep{shapira2009mode}.
The authors did notice a lot of advantages of median shift upon mean shift and make interesting experiments, however, they do not demonstrate its convergence.

Using $d(x_1,x_2) = \norm{x_1-x_2}_1$ ($L_1$-norm) and the triangular kernel (defined by $K(u) = \max(1-u,0)$), the p.d.f. of Eq. \eqref{eq::f_general} becomes:
$$
f(x) = \sum_{i \in C(x)} \left( 1- \norm{x-x_i}_1 \right),
$$
where $\norm{\cdot}_1$ denotes the $L_1$-norm in $\RR^q$ and $C(x) = \{ i \in [1,N], \norm{x-x_i}_1 < 1 \}$ is the set of  indices $i$ for which we have $K(\norm{x-x_i}_1) \not= 0$. Eq. \eqref{eq::iterate3} defining the generalized mean shift iteration becomes:
\begin{equation}
\label{eq::MedShift}
\hat{x}_{n+1} = \text{Med}( \{x_i, i \in C(\hat{x}_{n}) \}),
\end{equation}
with $C_n = {i \in I, \norm{\hat{x}_n-x_i}_1<1}$, and $\text{Med}(.)$ the term by term median of a set of points, as the median minimizes the L1 norm. The general result of Sect. \ref{sect::Main} applies and ensures convergence of this ``median shift'' algorithm
\begin{prop}
\label{prop::L_1_convergence}
The sequence of mode estimates returned by Eq. \ref{eq::MedShift} is convergent (more specifically it is stationnary) for any initialization, as a direct application of theorem \ref{thm::main}.
\end{prop}

%
%%%%%%%%%%%% ALGO MEDIAN SHIFT %%%%%%%%%%%%%%%%%%%
%
%\begin{algorithm}[H]
%\caption{Median shift}
%\label{algo::MS_L1}
%\begin{algorithmic}
%\FOR{$j=1:N$}
%\STATE $\hat{x}_0 = x_j$.
%\WHILE{$\hat{x}_n+1 \neq \hat{x}_n$}
%\STATE 
%\label{eq::rrr}
%\hat{x}_{n+1} = \text{Med} (\{ x_i, i \in C(\hat{x}_n) \}).$
%\ENDWHILE
%\ENDFOR
%\end{algorithmic}
%\end{algorithm}

%\begin{rem}
%The property shown in step 2 in the proof of Proposition \ref{prop::L_1_convergence} may look convoluted as in most cases, the set $C(\hat{x}_n)$ stops changing as soon as $f(\hat{x}_n)$ reaches its final value, meaning we have $C(\hat{x}_{n+1})$. But here, $C(\hat{x}_n)$ being decreasing for the inclusion cannot be replaced with $C(\hat{x}_n)$ being constant. Indeed, adapting the proof would require the number $D$ to be a sum of strictly negative (instead of non-positive) terms, which is not true due to the possibility for some points of the dataset to end up on the boundary of $C(\hat{x}_n)$. The weaker statement of step 2 allows handling this situation.
%\end{rem}

%% file: 3-2-Wasserstein_Median_Shift.tex
\subsection{Application 2: Wasserstein median shift}
\label{sect::WMS}

The Wasserstein metric, also called the Earth Mover's Distance, is a distance function between probability distributions on a set $U$, linked to optimal transport theory. 
We focus here on the specific case where $U$ is a finite set of cardinal $q$, and the probability distributions are approximated by normalized histograms.
We do not enter here in the generalized form of Wasserstein metric, as it is not the main point of the paper.
The intuition is that this distance it takes into account the order of the histogram's bins, as the difference of levels between close bins as less impact than the difference on far bins. 

Under these hypotheses, we can define the Wasserstein distance between two histograms $M$ and $N \in \mathcal{H}^q$, where $\mathcal{H}^q \subset (\RR_{\geq 0})^q$ is the set of normalized histograms: $\mathcal{H}^q = \{ z \in (\RR_{\geq 0})^q, \sum_{k=1}^q z^k =1 \} $, but it can be better defined in the space of cumulated histograms: $\mathcal{C}^q = \{ z \in (\RR_{\geq 0})^q, z^q=1, z^{k+1} \geq z^k \} $. 

Let us introduce two operators linking histograms to cumulated histograms:
\begin{equation}
\text{cumul}(x)_k = \sum_{i=1}^k x_i \quad  \text{and} \quad \text{diff}(z)_k = z_n - z_{n-1}
\end{equation}
with the convention $z_{-1}=0$. We have of course:
$$
\text{cumul}^{-1} = \text{diff}
$$

The Wasserstein metric can then be shown to take the simpler form below:
\begin{equation}
\label{eq::wasserstein_metric}
W_{1}(M,N) = ||\text{cumul}(M)-\text{cumul}(N)||_{L1}.
\end{equation}

A natural way of obtaining a minimizer for Eq. \ref{eq::iterate3} is to compute it in the space of cumulated histograms, then come back to the space of normalized histograms using the function diff. All we need to check is that the term-by-term median of a set of cumulated histograms is still a cumulated histogram.
\begin{prop}
The term-by-term median of a set of cumulated histograms is still a cumulated histogram.
\end{prop}
\begin{proof}
Let ${z_i, i \in J}$ be a set of cumulated histograms and $\bar{z}$ its term-by-term median defined by $\bar{z}^k = Med_i(z_i^k)$ (coordinates are denoted with a superscript). For $\bar{z}$ to be a cumulated histogram we have to check $\bar{z}_K=1$, $\forall k, \bar{z}^k \geqslant 0$ and that $\bar{z}$ is growing. The two first properties are obvious ( the median of numbers all equals to one is one, the median of positive numbers is positive). To show the third required property we notice that the ``median'' function is growing w.r.t. each of its arguments. Thus, having $(z_i^k \leqslant z_i^{k+1})$ for each $i$ implies $\text{Med}(z^k) \leqslant \text{Med}(z^{k+1})$. This is all we needed to establish that the median of a set of cumulated histograms is still a cumulated histogram.
\end{proof}

\begin{prop}
\label{prop::WMS_optimization}
If the distance function $d(\cdot,\cdot)$ in the equation \eqref{eq::iterate3} is the Wasserstein metric $W_1(\cdot,\cdot)$ defined by Equation  \eqref{eq::wasserstein_metric}, then the optimization step in \eqref{eq::iterate3} is solved by: 
\end{prop}
\begin{equation}
\label{eq::WMS_CD}
\hat{x}_{n+1} = \text{diff} ( \text{Med}( \{ \text{cumul}(x_i), i \in [1,N], \norm{ \text{cumul}(\hat{x}_n) - \text{cumul}(x_i)}_1  <1 \} )).
\end{equation}

An equivalent, but more practical formulation of the algorithm is the following:
\begin{equation}
\hat{x}_{n+1} = \text{diff} (\hat{z}_{n+1})
\quad  \text{with} \quad
\label{eq::unconstrained}
\hat{z}_{n+1} =  \text{Med}( \{ z_i, i \in C_n \} )
\end{equation}
where $C_n= \{ i \in I, \norm{\hat{z}_n - z_i}_1 <1 \}$. Proposition \ref{prop::WMS_optimization} provides a closed-form formula for the optimization problem appearing in Eq. \eqref{eq::iterate3} when used with the Wasserstein metric. It makes the adaptation of Mean Shift tractable in practice, but also suggests directly working on the cumulative histograms, which gives Algorithm \ref{algo::WMS} (Wasserstein median shift). Following Proposition \ref{prop::L_1_convergence}, the obtained estimate $\hat{x}_n$ is again stationary.

\begin{algorithm}
\caption{Wasserstein median shift}
\label{algo::WMS}
\begin{algorithmic}
\STATE Compute $z_i = \text{cumul}(x_i)$ for each $i$.
\FOR{$j=1:N$}
\STATE
$\hat{z}_0 = z_j$.
\WHILE{$\hat{x}_{n+1} \neq \hat{x}_n$}
\STATE
$\hat{z}_{n+1} = \text{Med} (\{ z_i ,  \left| \left| z_i - \hat{z}_n \right| \right|_ 1 <1 \}).$
\STATE
$\hat{x}_{n+1} = \text{diff}(\hat{z}_{n+1})$.
\ENDWHILE
\ENDFOR
\end{algorithmic}
\end{algorithm}

%% file: 4-Results.tex
\section{Experiment results}
\label{sect::Results}

In this section, we present results on two datasets, one on synthetic data, with a pure illustrative purpose, and one on real aeronautical data. 
We compare our wasserstein median shift algorithm with classical mean shift and other classical algorithms of clustering.
We do not claim that our algorithm works better in all cases, (a property which cannot be true for any clustering algorithm), but we show here some cases where it can outperfom classical methods.
The correlation between the obtained results and the ground truth is assessed using the Adjusted Rand Index (ARI \citep{rand1971objective}).
An ARI of 1 means that two clusterings are identical, an ARI close to 0 (or negative) expresses a total dissimilarity between the two clusterings. 

\subsection{Datasets}

%\subsubsection{Synthetic data}
To illustrate the flaws of the $L_2$ distance for histogram clustering, we build a data set consisting of empirical histograms, each computed using $100$ samples of a random variable.
The law of the random variable is different for each histogram, but belongs to one of the two following classes.
\begin{description}
\item[Class 1:] Gaussian variables having similar means and the same variance.
\item[Class 2:] Mixtures of two Gaussian variables, one of them chosen as in class 1.
\end{description}
For Class 1, the Gaussian have their means ranging from 0.47 to 0.53. For Class 2, the two components of the mixture have weights of 0.8 and 0.2 respectively. Their means range from 0.47 to 0.53 and from 0.17 to 0.23 respectively. The standard deviation is $0.02$ for all Gaussian.
For example, samples of these two classes of histograms are displayed on Figure \ref{fig::histograms_examples}, and, looking at it, retrieving the two classes is \emph{a priori} an easy task.

\begin{figure}
\center
\includegraphics[width=12.0cm]{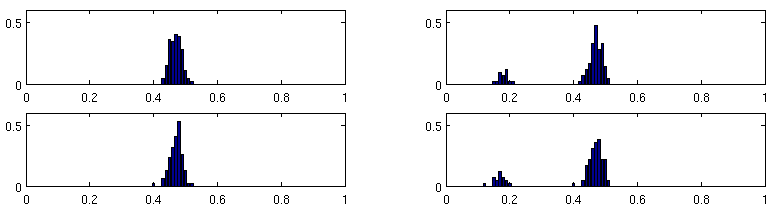}
\caption{Two samples of the two classes of histograms generated for the experiment described in Section \ref{sect::Results}. Left: class 1 ; right : class 2. We see their difference is obvious. Yet, the classical mean shift based on the $L_2$-norm performs very poorly on this data set, contrary to the Wasserstein median shift.}
\label{fig::histograms_examples}
\end{figure}

%\subsubsection{Real trajectory data}

The second situation we considered is the application framework which motivated the research presented in the paper. A fleet of helicopters is used by a company for several kinds of missions, and we look for a clustering algorithm able to retrieve the various use cases from in-flight recorded data. Instead of using long and complexe descriptors on temporal data, we used histograms, easier to compute on large sets of data, and more understandable.
The dataset we used, provided by an helicopter company, contains 122 time series representing the evolution, during a mission, of the altitude of a helicopter. They cannot be directly compared as vectors because their lengths are different. Note that the general issue of comparing time series having different sizes is difficult in itself, and several responses such as Dynamic Time Warping have been brought in the past. The solution we chose here is comparing the series through their histograms.

\begin{figure}
\center
\includegraphics[width=12cm]{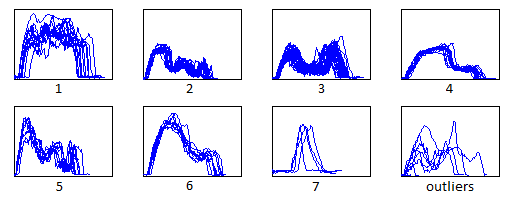}
\caption{True clusters of the dataset, and outliers in the last square.}
\label{fig::trajdata}
\end{figure}

\subsection{Experiments}
We tested 8 techniques from scikit-learn: Mini Batch K-means (MBKM), Affinity Propagation (AP), Spectral Clustering (SC), Agglomerative Clustering with Ward (ACW), Agglomerative Clustering with average linkage (ACA), DBSCAN (DB), Birch (Bir), and Gaussian Mixture (GM) and compare to Wasserstein Median Shift (ours, WMS) and Mean Shift (MeaS). Some of them being stochastic, we launched them 100 times and gave the standard deviation of the ARI. For all methods, when a parameter should be given, we gave to the algorithm reasonnable value (for example, we gave exact number of clusters for synthetic data). The results are given in the Table \ref{tab::results1}.
For both datasets, Wasserstein Median Shift outperforms all other clustering algorithms. We can also see that some algorithms perform better on real than synthetic data. This is due to our synthetic dataset being specifically designed to show the $L2$-norm can be tricked with mere Gaussian laws. The only exception is DBSCAN, which works for synthetic data and could not work on real ones (we tried with several $\epsilon$ value, but none worked). This result could be explain by the fact that in our real data there are samples linking clusters, making difficult for DBSCAN to separate them. 

\begin{table}[h!]
\label{tab::results1}
\begin{tabular}{|c|c|c| c|c|c| c|c|c| c|c|}
\hline
Data			& WMS 		  	& MeaS 	 &  MBKM    	  & AP 		& SC 	   	& ACW   & ACA 	  		& DB  	& 	Bir 	& GM \\
%\hline
%Synthetic (122)		& $\emph{1.0}$	& $0.33$ & $0.01\pm0.01 $ & $0.25$	& $0.11$	& $0.0$  & $0.0$  & $0.44$	& $0.0$		& $0.05\pm0.16$\\
\hline
Synth	& $\emph{1.0}$	& $0.11$ & $0.02\pm0.05 $ & $0.15$	& $0.0$		& $0.08$ & $0.0$  & $0.68$	& $0.0$		& $0.10\pm0.23$\\
\hline	
Real 			& $\emph{0.78}$	& $0.15$ & $0.38\pm0.04 $ & $0.49$	& $0.51$	& $0.67$ & $0.29$ & $0.0$		& $0.67$	& $0.43\pm0.06$\\
\hline
\end{tabular}
\caption{Results on real and synthetic dataset.}
\end{table}

We also went further and wrote a ``Wasserstein version'' of the algorithms for which it could be easily done: K-means-Wasserstein (KMWS) and DBSCAN-Wasserstein (DBSCAN WS).
All results are presented in the table \ref{tab::results2}
We can see that the Wasserstein versions of the algorithms work better on synthetic and real data than their $L_2$ counterparts.
Nevertheless, modified algorithms are still outperformed by Wasserstein Median Shift, which hypothesis fit better to the dataset.
Our understanding of the difficulties encountered with $L_2$-norm for histogram clustering is that the $L2$-norm is invariant by histogram bins permutations while the order of the bins is of great significance for histograms. 

\begin{table}
\label{tab::results2}
\begin{center}
\begin{tabular}{|c|c|c| c|c|c| }
\hline
Method 				& WMS 		  	& 	KMWS 		& DBSCAN WS 		\\
%\hline
%Synthetic (122)		& $\emph{1.0}$	& $1.0\pm0.0$	& $1.00$ 		\\
\hline
Synthetic	& $\emph{1.0}$	& $0.98\pm0.01$	& $0.99$ 		\\
\hline	
Real Data			& $\emph{0.78}$	& $0.51\pm0.07$ & $0.44$ 		\\
\hline
\end{tabular}
\caption{Results with only Wasserstein algorithms.}
\end{center}
\end{table}